\def\year{2022}\relax
\documentclass[letterpaper]{article} 
\usepackage{aaai22}  
\usepackage{times}  
\usepackage{helvet}  
\usepackage{courier}  
\usepackage[hyphens]{url}  
\usepackage{graphicx} 
\usepackage{natbib}  
\usepackage{caption} 
\DeclareCaptionStyle{ruled}{labelfont=normalfont,labelsep=colon,strut=off} 
\usepackage{algorithm}
\usepackage{algpseudocode}
\usepackage{textgreek}
\usepackage{blkarray}
\usepackage{booktabs}
\usepackage{amsmath}
\usepackage{amssymb}
\usepackage{amsthm}
\usepackage{tikz}
\newcommand*\circled[1]{\tikz[baseline=(char.base)]{
		\node[shape=circle,draw,inner sep=.1pt] (char) {#1};}}
\theoremstyle{plain}
\newtheorem{definition}{Definition}[section]   
\newtheorem{assumption}{Assumption}[section] 
\newtheorem{theorem}{Theorem}[section] 

\usepackage{newfloat}
\usepackage{listings}
\floatstyle{ruled}
\newfloat{listing}{tb}{lst}{}
\floatname{listing}{Listing}
\title{Covered Information Disentanglement:  Model Transparency via Unbiased Permutation Importance
}
\author{
   João Pereira\textsuperscript{\rm 1,2}, 
    Erik S.G. Stroes\textsuperscript{\rm 1},
    Aeilko H. Zwinderman\textsuperscript{\rm 1},
    Evgeni Levin\textsuperscript{\rm 1,2}
}
\affiliations{
    \textsuperscript{\rm 1}Amsterdam University Medical Center, \\Meibergdreef 9 1105 AZ, \\Amsterdam, The Netherlands \\
    \textsuperscript{\rm 2}Horaizon, Marshallaan 2 2625 GZ, Delft, The Netherlands \\
    \{j.p.belopereira, e.levin\}@amsterdamumc.nl
}

\usepackage{bibentry}

\begin{document}

\maketitle

\begin{abstract}
	Model transparency is a prerequisite in many domains and an increasingly popular area in machine learning research.
In the medical domain, for instance, unveiling the mechanisms behind a disease often has higher priority than the diagnostic itself since it might dictate or guide potential treatments and research directions. One of the most popular approaches to explain model global predictions is the \textit{permutation importance} where the performance on permuted data is benchmarked against the baseline. However, this method and other related approaches will undervalue the importance of a feature in the presence of covariates since these cover part of its provided information. To address this issue, we propose Covered Information Disentanglement (\textit{CID}), a method that considers all feature information overlap to correct the values provided by \textit{permutation importance}. We further show how to compute \textit{CID} efficiently when coupled with \textit{Markov random fields}.  We demonstrate its efficacy in adjusting \textit{permutation importance} first on a controlled toy dataset and discuss its effect on real-world medical data.
\end{abstract}

\section{Introduction}

Understanding the biological underpinnings of disease is at the core of medical research.  Model transparency and feature relevance are thus a top priority to discover new potential treatments or research directions.
One of the current most popular methods to explain local model predictions is \textit{SHAP} ~\cite{shap_reg,shap_2,shap_3}, a game-theoretic approach that considers the features as ``players" and measures their marginal contributions to all possible feature subset combinations. \textit{SHAP} has also been generalized in \textit{SAGE} \cite{SAGE} to compute global feature importance. However, recent work by Kumar et al. \cite{shap_prob} exposes some mathematical issues with \textit{SHAP} and concludes that this framework is ill-suited as a general solution to quantifying feature importance. 
Other local-based methods such as \textit{LIME} \cite{LIME} and its variants (see e.g. \cite{LIME2,LIME3,LORE,gse}) build weak yet explainable models on the neighborhood of each instance. While this achieves higher prediction transparency for each data point, in this work, we are mainly concerned with a more holistic view of importance, which may be more appropriate to guide new research directions and unravel disease mechanisms. 
Tree-based methods are very commonly selected for this purpose because they compute the impurity or  \textit{Gini importance}  \cite{rf}. The impurity importance is biased in favor of variables with many possible split points; i.e. categorical variables with many categories or continuous variables \cite{bias_rf}. 
A generally accepted alternative to computing the  \textit{Gini importance}  is the \textit{permutation importance} \cite{fisher}, which benchmarks the baseline performance against permuted data. There is, however, the issue of multicollinearity. When features are highly correlated, feature permutation will underestimate the individual importance of at least one of the features, since a great deal of the information provided by this feature is ``covered" by its covariates. One option is to permute correlated features together \cite{tolosi}. However, this implies choosing an arbitrary correlation grouping threshold or performing cross-validation to determine the optimal number of groups that yield the best estimator, resulting in slow running times. Most importantly, it leaves out the differentiation of their contributions to the final prediction. Motivated by the idea that there is an information overlap between different features, we develop Covered Information Disentanglement (\textit{CID}),\footnote{We make an implementation of \textit{CID} publicly available at: https://github.com/JBPereira/CID.} an information-theoretic approach to disentangle the shared information and scale the \textit{permutation importance} values accordingly. We demonstrate how \textit{CID} can recover the right importance ranking on artificial data and discuss its efficacy on the Cardiovascular Risk Prediction dataset \cite{ehj_proteomics}.

\section{Methodology}
\subsubsection*{Notation}
We denote matrices, 1-dimensional arrays, and scalars/functions with capital bold, bold, and regular text, respectively (e.g. $\mathbf{X},\; \mathbf{x},\; \alpha/f$). Given a dataset $\mathbf{X}_{M\times N}$, we will denote its random variables by capital regular text with a subscript and the values using lowercase (e.g. $X_i$ and $x_i$), while the joint density/mass will be represented as $p(x)$. The expected loss of a function given by: $\frac{1}{M}\sum_{i=1}^{M}l\left[y,\,f(\mathbf{x}_i)\right]$ will be denoted by $\mathcal{L}\left[f\left(\mathbf{X}\right)\right]$.

\subsection{Information Theory background}\label{section:it}
\textit{Information theory} (IT) is a useful tool used in quantifying relations between random variables. The basic building block in IT is the \textit{entropy} of an r.v. $X_i$, which is defined as: $H(X_i)\equiv -\sum_{x_i}p(x_i)\,\text{log}\, p(x_i).$
The \textit{joint entropy} between r.v.s $X_i$ and $X_j$ is defined as: $H(X_i, X_j)\equiv -\sum_{x_i}\sum_{x_j}p(x_i, x_j)\text{log}\, p(x_i,x_j).$
The \textit{mutual information} between r.v.s $X_i$ and $X_j$ is the relative entropy between the joint entropy and the product distribution $p(x_i)p(x_j)$: $I(X_i, X_j)\equiv \sum_{x_i}\sum_{x_j}p(x_i, x_j)\,\text{log}\, \frac{p(x_i,x_j)}{p(x_i)p(x_j)}.$
For a more thorough exposition to IT, the reader can refer to \cite{IT_book}.\\
Using the definitions above, one can derive properties that resemble those of set theory, where joint entropy and mutual information are the information-theoretic counterparts to union and intersection, respectively \cite{info_amount}.
In order to keep this intuition when generalizing to higher dimensions, one can define the entropy of the union of $N$ features as:
\begin{definition}{\textbf{Multivariate Union Entropy}}\label{union_entropy}
	\begin{equation*}
	H\left(\cup_{i=1}^{N}X_i\right)\equiv-\sum_{x_i}p(x_1,\,...,\,x_N)\text{log}\,p(x_1,\,...,\,x_N)
	\end{equation*}
\end{definition}
and using the Inclusion-Exclusion principle, we can define the intersection as:
\begin{definition}{\textbf{Multivariate Intersection Entropy}}\label{int_entropy}	
	\begin{align*}
	&H\left(\cap_{i=1}^{N}X_i\right)\equiv \sum_{x_1,...,x_N}p(x_1,\,...,x_N)h(x_1,\,...,x_N),\\
	&h(x_1,\,...,x_N)=\sum_{k=1}^{N}(-1)^{k-1}\mathop{\sum_{I\subseteq\{1,\,...,\,N\};}}_{|I|=k}h(x_{I_1},\,...\,,\,x_{I_{k}}),
	\end{align*}
where $h(\mathbf{x})=-log\,p(\mathbf{x})$ is the \underline{local entropy}.
\end{definition}
This definition of multivariate intersection is also called co-information and it may yield negative values. To see this, consider the case of three sets of r.v.s $X_i$, $X_I$, and $Y$ and suppose there is no correlation between $X_i$ and $X_I$. If we rewrite the mutual information expression into $I(X_i, Y)-I(X_i, Y|X_I)$, then this expression may become negative when the information provided by $X_i$ and $Y$ given a fixed value of $X_I$ is higher than that of $I(X_i, Y)$. This can happen for instance if $X_i$ has no correlation with $X_I$ but knowing $X_I$ introduces a correlation between the two (what is commonly known as `explaining away'). This motivated Williams and Beer to draw the distinction between redundant and synergistic information and propose \textit{partial information decomposition} (PID) \cite{beer}. Ince \cite{Ince} thoroughly analyzed the multivariate properties of PID directly applied to multivariate entropy and suggested to divide the individual terms in definition \ref{int_entropy}, so that positive local entropy terms correspond to redundant entropy, while the negative ones correspond to synergistic entropy.

\subsection{Permutation Feature Importance}
Feature importance is a subjective notion that may vary with application. 
Consider a supervised learning task where a model $f$ is trained/tested on dataset $\mathbf{X},\,\mathbf{y}$ and its performance is measured by a function $\mathcal{L}$. 
In this work, we will refer to feature importance as the extent to which a feature $X_i$ affects $\mathcal{L}[f(\mathbf{X})]$, on its own and through its interactions with $X_{\backslash\{i\}}$.
Permutation importance was first introduced by Breiman \cite{rf} in random forests as a way to  understand the interaction of variables that is providing the predictive accuracy. \\
Consider a dataset $\mathbf{X}_{M\times N}$ and denote the $j$th instance of the $i$th feature by $\mathbf{X}_{i}^{j}$.
Suppose the set $\{1,\,...,M\}$ is sampled and denote the subsample by $\mathbf{s},\,\mathbf{s}\subseteq\{1,\,...,M\}$. Consider further a random permutation of this subset which we denote by $\boldsymbol{\pi}\left(\mathbf{s}\right)$ and its $j$th element by $\boldsymbol{\pi}_j \left( \mathbf{s} \right )$. The \textit{permutation importance}, is given by:
\begin{alignat}{2}
&e_{i}(f, \mathbf{s})=\sum_{j\in \mathbf{s}}^{|\mathbf{s}|}
\Biggl ( &&
	\mathbf{E}_{\sim p(\boldsymbol{\pi})} 
	\left[
		\mathcal{L}
		\left(
			f\left(
			\mathbf{X}_{1}^{j},...,\mathbf{X}_{i}^{\boldsymbol{\pi}_j \left( \mathbf{s} \right ) },...,\mathbf{X}_{N}^{j}
		     \right) 
		\right)
	\right] 
	\nonumber\\& &&-\mathcal{L}
		\left(
		      f\left(
		      \mathbf{X}_{1}^{j},...,\mathbf{X}_{N}^{j}
		      \right) 
		\right) 
\Biggr )
\\
&e_{i}(f) = \mathbf{E}_{\sim p(\mathbf{s})}&&\left[e_{i}(f, \mathbf{s})\right] 
\end{alignat}

\subsection{Covered Information Disentanglement}

\begin{figure}[tb]
	\centering
	\includegraphics[width=0.6\linewidth]{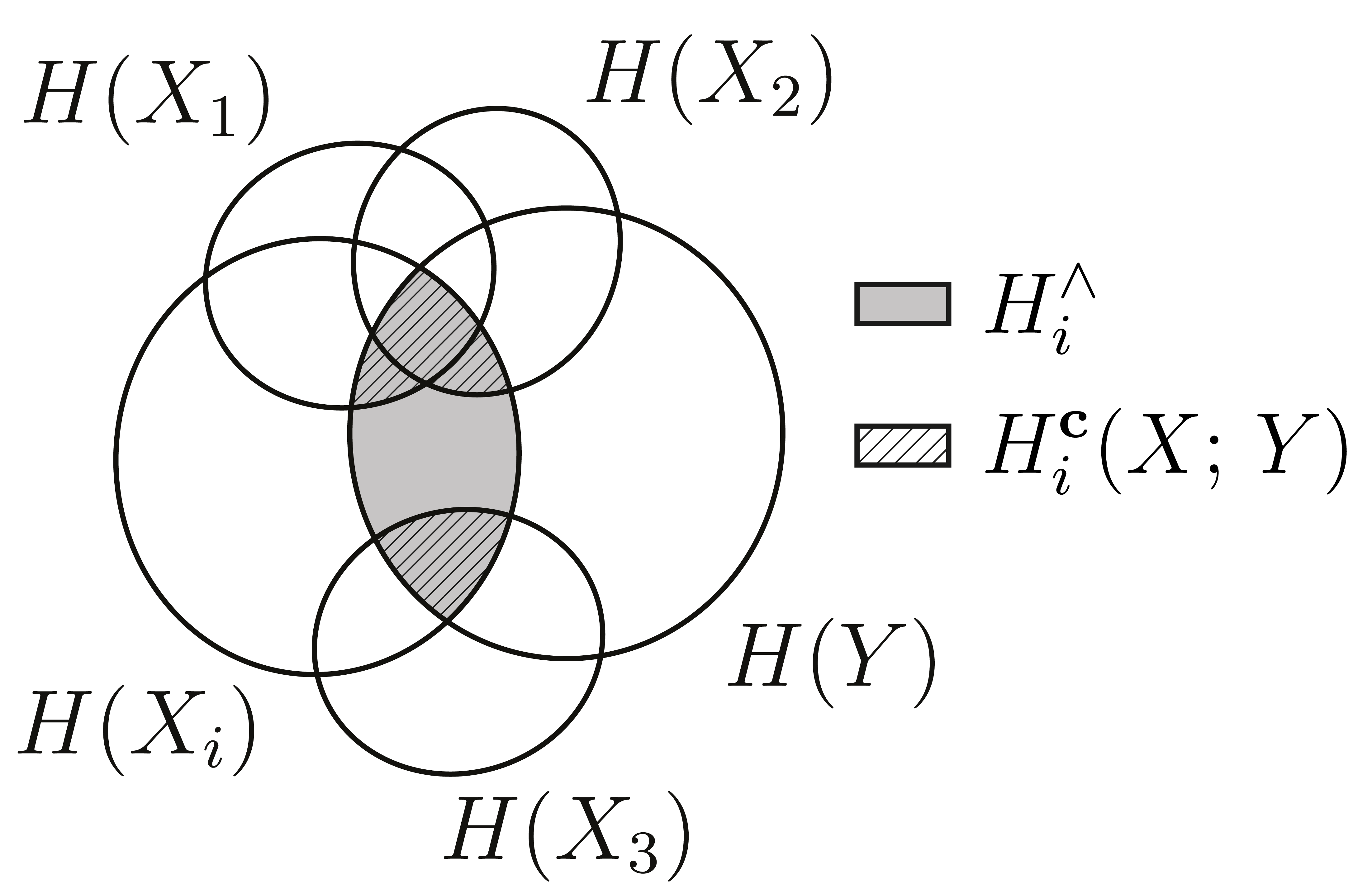}
	\caption{An illustration of the \textit{permutation importance} bias in the presence of covariates and the measures needed to correct it. The mutual information between random variable $X_i$ and $Y$ (represented in gray) is covered by the information provided by r.v.s $X_1$, $X_2$ and $X_3$. Permutation importance only measures the non-covered part (non-shaded gray), and to correct its value, we suggest computing $H_{i}^{\mathbf{c}}(X;\,Y)$. }
	\label{covered_info_pic}
\end{figure}

In the presence of covariates, the \textit{permutation importance} measures the performance dip caused by removing the non-mutual information between the feature and the remaining data. That is:
\begin{equation}\label{eq_target}
e_{i}\left(f\right)  = \mathcal{I}_i\left(f\right)-e_{i}^\mathcal{\cup}(f),
\end{equation}
where $\mathcal{I}_i\left(f\right)=\mathbf{E}_{\sim p(\mathbf{s})}\left[\mathcal{I}_i(f, \mathbf{s})\right]$ is the expected total importance of feature $i$ under model $f$ (the quantity we are interested in) and $e_i^\mathcal{\cup}\left(f\right)=\mathbf{E}_{\sim p(\mathbf{s})}\left[e_{i}^{\cup}(f, \mathbf{s})\right] $ is the expected performance dip covered by all other variables.
To compute $e_i^\mathcal{\cup}\left(f\right)$ would require applying the Inclusion-Exclusion principle and measuring the performance dip for all possible feature combinations of size $1$ to the number of features. Instead, we note that $e_i^\mathcal{\cup}\left(f\right)$ intuitively measures the model performance dip when the model is deprived of the information covered by the r.v.s that are correlated with $X_i$. For an intuitive depiction of the problem, see figure \ref{covered_info_pic}.\\
Motivated by the analogy between set-theory and information measures, we define the joint information between an r.v. and the target variable that is ``covered" by the other r.v.s as:
\begin{definition}\label{cov}
	{\textbf{Covered information (CI)}}
	Given an r.v. $X_i$ and a set of distinct r.v.s  $X_{\mathbf{i^-}},\,\mathbf{i^-}=\{1,\,...\,N\}\backslash \{i\}$, the information of $X_i$ w.r.t. $Y$ covered by $X_{\mathbf{i^-}}$  is defined as:
	\begin{equation*}
	H_{i}^{\mathbf{c}}(X;\,Y)=H\left(X_i\cap Y\cap\left\{\cup_{j\in \mathbf{i^-}}X_j\right\}\right).
	\end{equation*}
\end{definition}
When it is clear from the context what $Y$ and $X_{\mathbf{i^-}}$ are, we will abbreviate $H_{i}^{\mathbf{c}}(X;\,Y)$ into $H_{i}^{\mathbf{c}}$, denote the mutual information with $Y$ by $H_{i}^{\wedge}$, and the respective local entropy terms for the $k$th row in the dataset with $h_{ik}^{\mathbf{c}}\equiv h_i^{\mathbf{c}}(\mathbf{X}_i^k,Y^k)$ and $h_{ik}^{\wedge}\equiv h_i^{\wedge}(\mathbf{X}_i^k,\mathbf{y}^k)$.
We further divide $H_{i}^{\mathbf{c}}$ and $H_{i}^{\wedge}$ into its redundant and synergistic counterparts, which for a specific sample $\mathbf{s}$ are given by: 
\begin{align*}
&\textbf{Redundant MI}: H_{i}^{\wedge^+}(\mathbf{s}) = \frac{1}{|\mathbf{s}|}\sum_{k\in\mathbf{s}} \text{max}\left(0, h_{ik}^{\wedge}\right)\\
&\textbf{Synergistic MI}: H_{i}^{\wedge^-}(\mathbf{s}) = \frac{1}{|\mathbf{s}|}\sum_{k\in\mathbf{s}} \left|\text{min}\left(0, h_{ik}^{\wedge}\right)\right|\\
&\textbf{Redundant CI}: H_{i}^{\mathbf{c}^+}(\mathbf{s}) = \frac{1}{|\mathbf{s}|}\sum_{k\in\mathbf{s}} \text{max}\left(0, h_{ik}^{\mathbf{c}}\right)\\
&\textbf{Synergistic CI}: H_{i}^{\mathbf{c}^-}(\mathbf{s})  = \frac{1}{|\mathbf{s}|}\sum_{k\in\mathbf{s}} \left|\text{min}\left(0, h_{ik}^{\mathbf{c}}\right)\right|
\end{align*}
\begin{assumption}\label{core_assump}
	Permutation importance and entropy terms are related through a map $\phi_f:\mathbb{R}^4\rightarrow \mathbb{R}$, such that $e_{i}(f, \mathbf{s})=\phi_f\left(H_{i}^{\mathbf{c}^+}(\mathbf{s}), H_{i}^{\mathbf{c}^-}(\mathbf{s}),H_{i}^{\wedge^+}(\mathbf{s}),   H_{i}^{\wedge^-}(\mathbf{s})\right) + \epsilon$, where $\epsilon$ is an error term.
\end{assumption} 
Thus, if assumption \ref{core_assump} holds, we can use the information of $X_i$ w.r.t. $Y$ by $X_{\mathbf{i}^-}$  and approximate equation \ref{eq_target} with:
\begin{align}
e^\mathcal{\cup}_{i}(f,\mathbf{s})\approx  &\phi_f\left(0, H_{i}^{\mathbf{c}^-}(\mathbf{s}),H_{i}^{\wedge^+}(\mathbf{s}),   H_{i}^{\wedge^-}(\mathbf{s})\right) \;- \nonumber\\& \phi_f\left(H_{i}^{\mathbf{c}^+}(\mathbf{s}), H_{i}^{\mathbf{c}^-}(\mathbf{s}),H_{i}^{\wedge^+}(\mathbf{s}),   H_{i}^{\wedge^-}(\mathbf{s})\right).
\end{align}

This means we can approximate the result of permuting all possible combinations of features by computing only the single-feature permutation loss and the covered information of r.v. $X_i$ by all the others. Here, we are implicitly defining: $\mathcal{I}_i\left(f, \mathbf{s}\right)\equiv \phi_f\left(0, H_{i}^{\mathbf{c}^-}(\mathbf{s}),H_{i}^{\wedge^+}(\mathbf{s}),   H_{i}^{\wedge^-}(\mathbf{s})\right) $, and thus the true importance in the performance difference scale is given by mapping the entropy values when there is no redundant entropy to the space of performance differences. \\
Since we are predicting the feature importance using a map between entropy terms (which measure model-agnostic importance) and \textit{permutation importance} values, the end result depends only on how learnable is the model behavior w.r.t to entropy. Moreover, since the entropy values are computed for the different subsample sets $\mathbf{s}$, the overall importance variability is also estimated.\\
\begin{figure*}[tb]
	\centering
	\makebox[\textwidth][c]{
		\includegraphics[ width=.8\linewidth]{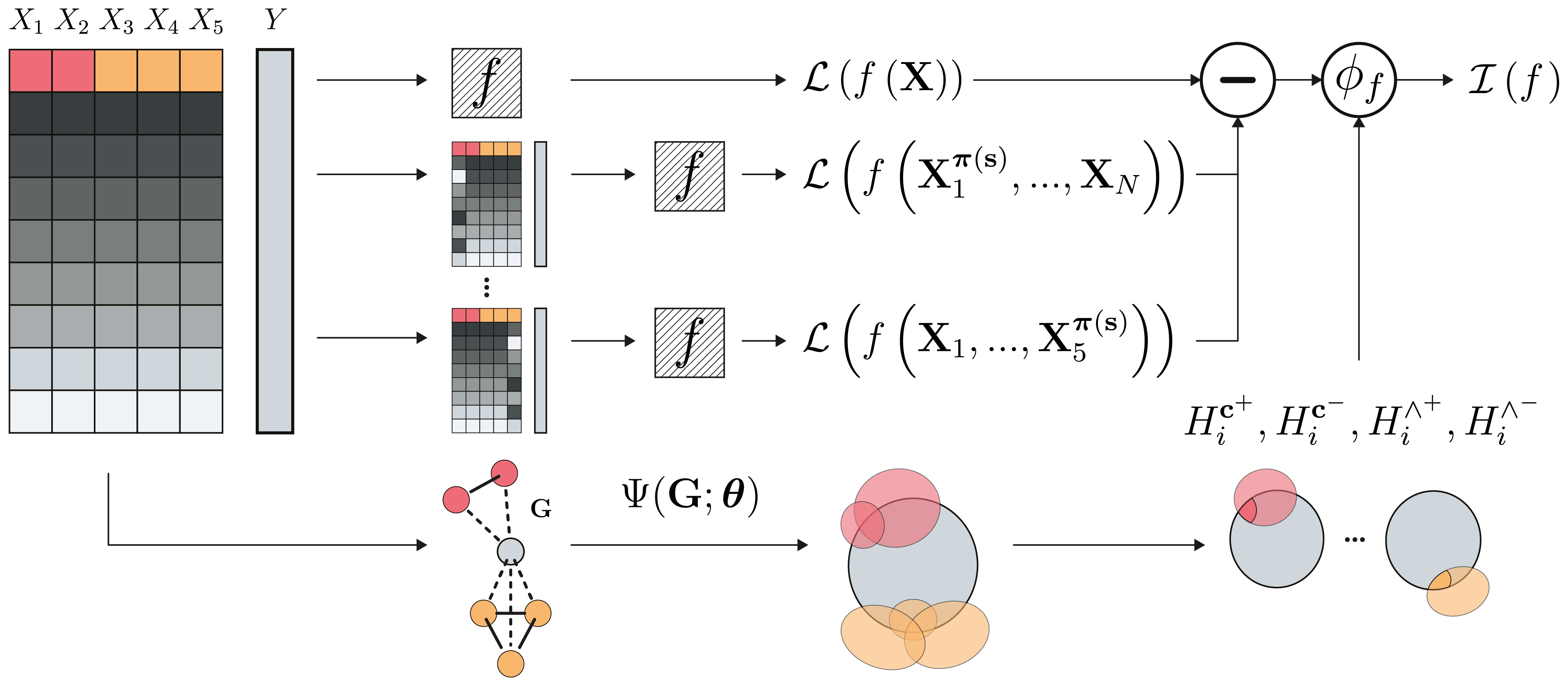}}
	\caption{\textit{CID} importance diagram. The permutation feature importance is computed by first calculating the expected loss of the model $f$ ($\mathcal{L}\left([f(\mathbf{X})\right)$). Then, each feature's values are permuted and the expected loss of $f$ computed. Subtracting each permuted dataset loss to the original one yields the \textit{permutation importance}. \textit{CID} starts by inferring the network $\mathbf{G}$ for the Markov Random Field $\Psi$ (alternatively, a prior network is given), then the MRF parameters $\boldsymbol{\theta}$ are inferred, and finally, $H_{i}^{\mathbf{c}}$/$H_i^{\wedge}$ are computed for each feature, which are then used to train the entropy/PI model $\phi_f$ and predict the true importance $\mathcal{I}(f)$.}
	\label{CID_flow}
\end{figure*}
There is still the issue of computing $H_{i}^{\mathbf{c}}$, since it involves computing $p(X)$. Since directionality is irrelevant for the purpose of computing overlapping information, we suggest to model $p(X)$ using an undirected graphical model (UGM).
Let $G=(V,\,E)$ denote a graph with $N$ nodes, corresponding to the $\{X_1,\,...,\,X_N\}$ features, and let $\mathcal{C}$ be a set of cliques (fully-connected subgraphs) of the graph $G$. Denoting a set of clique-potential functions by $\{\psi_{\mathcal{C}}:\,\mathcal{X}^{|\mathcal{C}|}\rightarrow \mathbb{R}\}$, the distribution of a Markov random field (MRF) \cite{pgm} is given by: $p(x)=\prod_{c\in \mathcal{C}}\psi_c(x_c)/ \mathbf{Z}$, where $\mathbf{Z}=\int \prod_{c\in \mathcal{C}}\psi_c(x_c) dx$ is the partition function. By the Hammersley-Clifford theorem, any distribution that can be represented in this way satisfies: $X_i\perp X_j|X_{\mathcal{N}(X_i)}$ for any $X_j\notin \mathcal{N}(X_i)$, where $\mathcal{N}(X_i)$ is the set $\{X_k:\, (i,k)\in E\}$. 
This allows to significantly simplify the expression of covered information yielding the main result of this paper:
\begin{theorem}\label{mrf_theorem} 
	Consider an r.v. $X_i$ and set of r.v.s $X_{\mathbf{i^-}},\;\mathbf{i^-}=\{1,\,...,N\}\backslash\{i\}$, a response r.v. $Y$, as well as the set of r.v.s that are neighbors to both $X_i$ and $Y$: $X_{\mathcal{N}(i,y)},\,\mathcal{N}(i,y)\in \cup\{\mathcal{N}(X_i), \,\mathcal{N}(Y)\}$.
	For a Markov Random Field, the covered information of $X_i$ by $X_{\mathbf{i^-}}$ w.r.t. $Y$ is given by:
	\begin{align}\label{main_eq}
	H_{i}^{\mathbf{c}} =H_{i}^{\wedge}-\mathbf{E}_{\sim p(x_{\mathcal{N}(i,y)})}\left[log\left(f\frac{\mathbf{d}^{T}\mathbf{F}\,\mathbf{e}}{\mathbf{d}^{T}\mathbf{F}_{y}\mathbf{F}_{x_i}^{T}\mathbf{e}}\right)\right],\nonumber
	\end{align}
	where $\mathbf{F}$ is a matrix with the product of joint potential values $\psi_{\mathcal{C}_F}$ for set of cliques $F:\,\{c\,|\,X_i, \,Y\in c\}$; $f$, $\mathbf{F}_y$ and $\mathbf{F}_{x_i}$ are an entry, column, and row of $\mathbf{F}$, respectively, while $\mathbf{d}$ and $\mathbf{e}$ are arrays with the product of potential values $\psi_{\mathcal{C}_D},\,\psi_{\mathcal{C}_E}$ for set of cliques $D:\,\{c\,|\,X_i\in c, \,Y\notin c\}$ and  $E:$ $\,\{c\,|\,X_i\notin c,\,Y\in c\}$ with fixed $X_{\mathbf{i}^-}$.
\end{theorem}
\begin{proof}
	Using definition \ref{union_entropy}, \ref{int_entropy} and \ref{cov}:
	\begin{align*}
	H_{i}^{\mathbf{c}}=&  H_{i}^{\wedge}+ \overbrace{H(X_i\cup Y\cup X_{\mathbf{i^-}})}^{\circled{1}}-\overbrace{H(X_{\mathbf{i^-}}\cup Y)}^{\circled{2}}\\&+\overbrace{H(X_{\mathbf{i^-}})}^{\circled{3}}-\overbrace{H(X_i\cup X_{\mathbf{i^-}})}^{\circled{4}}.
	\end{align*}
	The probability density for Markov Random fields is equal to $p(x)=\prod_{c\in\mathcal{C}}\psi_c(x_c)/ \mathbf{Z}$, where $\mathbf{Z}$ is the partition function and $\mathcal{C}$ is the set of cliques in the Markov network. Define two sets of cliques: $A:\{c\,|\,X_i \in c\}$ and $B:\{c\,|\,X_i\notin c\}$. In that case (ignoring the partition function term because it cancels out):
	\begin{align}
	&\circled{1}=-\sum_{x}p(x)\left[log \prod_{b\in B}\psi_{b}(x_{b})+
	log\prod_{a \in A}\psi_{a}(x_{a}) \right] ,\nonumber
	\\
	&\circled{2}=-\sum_{x}p(x)\left[log \prod_{b \in B}\psi_{b}(x_{b}) + 
	log\sum_{x_i}\prod_{a\in A}\psi_{a}(x_{a}) \right ] ,\nonumber
	\\
	&\circled{1}-\circled{2}=-\sum_{x}p(x)log\left(\frac{\prod_{a\in A}\psi_{a}(x_{a})}{\sum_{x_i}\prod_{a\in A}\psi_{a}(x_{a})}\right).\nonumber
	\end{align}
	
	To compute $\circled{3}-\circled{4}$, define four sets of cliques: $C:\{c\,|\,X_i \notin c,\,Y \notin c\}$, $D:\{c\,|\,X_i \in c,\,Y \notin c\}$, $E:\{c\,|\,X_i \notin c,\,Y \in c\}$, and $F:\{c\,|\,X_i \in c,\,Y \in c\}$. In order to reduce the clutter, we will introduce the following functions: $d(x_i, x_{\mathbf{i^-}})=\prod_{j\in \mathbf{i^-}, j\sim i}\psi(x_i, x_j)$, $e(y, x_{\mathbf{i^-}})=\prod_{j\in \mathbf{i^-}, j\sim y}\psi(y, x_j)$, $f(x_i, y)=\psi(x_i, y)$, where we will abbreviate $d(x_i, x_\mathbf{i^-})$ into $d(x_i)$ and $e(y, x_\mathbf{i^-})$ into $e(y)$ when the value for random variable $X_\mathbf{i^-}$ is fixed. Then (again, ignoring the partition function):
	
	\begin{align}
	&\resizebox{1\linewidth}{!}{$\circled{3}=-\sum_{x}p(x)\left [ log \prod_{c\in C}\psi_{c}(x_{c}) +  log\sum_{x_i}\sum_{y} d(x_i)e(y)f(x_i, y) \right ]$},\nonumber
	\\
	&\resizebox{1\linewidth}{!}{$\circled{4}=-\sum_{x}p(x)\left[ log \prod_{c\in C}\psi_{c}(x_{c}) + log\sum_{y} d(x_i)e(y)f(x_i, y) \right]$},\nonumber\\
	&\circled{3}-\circled{4}=-\sum_{x}p(x)log\left(\frac{\sum_{x_i}\sum_{y} d(x_i)e(y)f(x_i, y)}{\sum_{y} d(x_i)e(y)f(x_i=X_i, y)}\right), \nonumber
	\end{align}
	where $f(x_i=X_i, y)$ is the function $f$ for a fixed value of the r.v. $X_i$.
	Since the set of cliques $A=\{D\cup F\}$, and denoting by $d(X_i)$, $f(X_i, Y)$ the functions $d$ and $f$ for fixed values of $X_i$ and $Y$, then:
	\begin{align}
	\begin{split}
	&(\circled{1}-\circled{2}) + (\circled{3}-\circled{4}) =\\ &\resizebox{1\linewidth}{!}{$-\sum_{x}p(x)log\left(\frac{\sum_{x_i}\sum_{y}d(X_i)d(x_i)f(X_i, Y)e(y)f(x_i, y)}{\sum_{x_i}\sum_{y}d(X_i)d(x_i)f(x_i, Y)e(y)f(X_i, y)}\right)$}\\
	&=-\mathbf{E}_{\sim p(x_{\mathcal{N}(i,y)})}\left[log\,f(X_i, Y) + log\,\left(\frac{\mathbf{d}^{T}\mathbf{F}\,\mathbf{e}}{\mathbf{d}^{T}\mathbf{F}_{y}\mathbf{F}_{x_i}\mathbf{e}}\right)\right]\nonumber, 
	\end{split}
	\end{align}
	where $x_{\mathcal{N}(i,y)}$ is an instance of the set of r.v.s that are neighbors to either $X_i$ or $Y$, $\mathbf{d}$ and $\mathbf{e}$ are column arrays with the different values of $d(x_i)$  and $e(y)$ for fixed $X_\mathbf{i^-}$, $\mathbf{F}$ is a matrix with all the values $f(x_i, y)$ with varying values of $X_i$ in the rows and $Y$ in the columns, and $\mathbf{F}_y$ and $\mathbf{F}_{x_i}$ are row and column vectors of $\mathbf{F}$ corresponding to fixed $Y$ and fixed $X_i$, respectively. This yields the result of the theorem.
\end{proof}

\subsubsection{Considerations and simplifications}
If a $2$-clique MRF is chosen, then $\mathbf{F}$ depends only on $X_i$ and $Y$, and can be computed before the expectation. \\
\textbf{Gaussian MRF:}
Learning an MRF's network structure is expensive. One popular approach is to use \textit{graphical lasso} \cite{graphical_lasso} which learns the entries of a Gaussian precision matrix by finding: $\underset{\boldsymbol{\Lambda} \in \mathbb{S}_{+}^{n}}{\text{min}}-log\;det(\boldsymbol{\Lambda}) + tr(\mathbf{S}\boldsymbol{\Lambda})+\rho||\boldsymbol{\Lambda}||_1$, where $\boldsymbol{\Lambda}$ is the precision matrix (constrained to belong to $\mathbb{S}_{+}^{n}$, the set of positive semi-definite $n\times n$ matrices), $\mathbf{S}$ is the empirical covariance matrix and $\rho$ acts in analogy to Lasso regularization by penalizing a large number of non-zero precision entries. We can model the potentials using Gaussian Markov random fields whose potentials are $\psi_{s, t}(x_s, x_t)=\text{exp}\left[-\frac{1}{2}x_s\Lambda_{st}x_t\right],\; \psi_{s}(x_s) = \text{exp}\left[-\frac{1}{2}\left(x_s^{2}\Lambda_{ss}+2\eta_sx_s \right)\right]$, where $\boldsymbol{\eta}=\mathbf{\Lambda}\boldsymbol{\mu}$ ($\boldsymbol{\mu}$ is the mean vector).\\ 
\textbf{Discrete Approximation:}
Continuous MRF such as Gaussian Markov Random fields depend on a continuous multivariate distribution and thus the entropy must be replaced by differential entropy, which violates many of the desired properties of discrete entropy. Therefore, we will approximate a continuous distribution with a discrete one $p(x_i)\approx\delta_i p(\overline{x_i})$, where $\delta_i$ is the $i$th feature bin size and $\overline{x_i}$ is the mean value of the bin, and then carry on with our computations as specified in theorem \ref{mrf_theorem}. For the case where all bins have the same size per feature, all the $\delta$s cancel out. \\
\textbf{Complexity:}
If we compute the expectation in theorem \ref{mrf_theorem} as the empirical expectation, then the asymptotic complexity becomes $\mathcal{O}(SB^2)$, where $S$ is the number of samples taken in the empirical expectation and $B$ is the maximum between the number of bins used to discretise continuous values and the maximum number of values the discrete features take.\\
The covered information computation for each feature can be done in parallel.
\section{Experimental Section}
To test the  \textit{CID} ranking adjustment, we first tested it on a toy dataset where the real importances are known, and a real-world medical dataset. We implemented \textit{CID} in Python using scikit-learn's \textit{graphical lasso} \cite{scikit}. 
For the toy dataset, we used scikit-learn's Extremely Randomized Trees and Bayesian Regression implementations, and for the medical dataset we used a Gradient Boosting Survival model \cite{sksurv}. 
\subsection{Multivariate Generated Data Test}\label{mvn_dataset}
In order to test if  \textit{CID} adjusts the permutation ranking into the correct one, we took $800$ samples from a multivariate distribution with the following marginal distributions:
$X_1\sim \text{Gamma(2,2)}$, $X_2\sim\text{Beta}(0.5,0.5)$, $X_3\sim X_1\cdot X_2$, $X_4\sim -\text{Exponential}(0.2)$, $X_5\sim \text{sin}(X_4)$ and $X_6\sim X_7\cdot X_8 + \left(1-X_7\right)\cdot X_9$ with $X_7\sim \text{Bin(1, 0.7)}$ and $X_8\sim \mathcal{N}(-5, 1)$, $X_9\sim \mathcal{N}(5, 1)$. We then defined the outcome variable as:
\begin{equation}
y=\begin{cases}
X_1, & \text{if}\; u\leq 0.15\\
X_2, & \text{if} \;0.15\leq u\leq 0.3\\
X_3, & \text{if} \;0.3\leq u\leq 0.5\\
X_1 + X_2+X_3, & \text{if} \;0.5\leq u\leq 0.65\\
X_4, & \text{if} \;0.65\leq u\leq 0.75\\
X_5, & \text{if}\; 0.75\leq u\leq 0.85\\
X_4+X_5, & \text{if} \;0.85\leq u\leq 0.95\\
X_6, & \text{if}\;u\geq 0.95\\
\end{cases}, 
\end{equation}
where $u$ is an observation of $U\sim \text{Uni}(0,1)$. The true importances are thus: $\mathcal{I}_3\geq\mathcal{I}_1=\mathcal{I}_2\geq \mathcal{I}_5=\mathcal{I}_4\geq\mathcal{I}_6$
We transformed the data into Gaussian using quantile information and chosen gaussian markov random fields to pair with \textit{CID}. The graph was inferred using graphical lasso with a grid-search cross-validation to determine the optimal $l_1$ penalization parameter. To test the  \textit{CID} correction, we performed $200$ Shuffle Splits with Extremely Randomized Trees and computed the  \textit{Gini importance} for each feature, as well as the \textit{permutation importance}. We then adjusted the feature importances using the  \textit{CID} algorithm and Bayesian Regression as $\phi$ (see assumption \ref{core_assump}). You can compare the rankings in figure \ref{MV_comparison_pic}.
As can be seen from the swarmplot in figure \ref{MV_comparison_pic}, with the exception of $X_3$, \textit{permutation importance} placed a nearly equal weight on all features, centered around zero, presumably due to the high feature covariance. The  \textit{CID} was able to rectify this ranking and ranked the features in the right order. It also placed every feature importance at non-zero with a clear gap between unequally important features and similar importance for $X_1$/$X_2$ and $X_4$/$X_5$, matching well the true importances. Moreover, notice how the \textit{Gini importance} underestimated $X_2$, presumably because $X_3$ and $X_1$ offer nearly as good partitions as $X_2$ due to their overlap and similarity.

\begin{figure}[tb]
	
	\centering
	\includegraphics[ width=.95\linewidth]{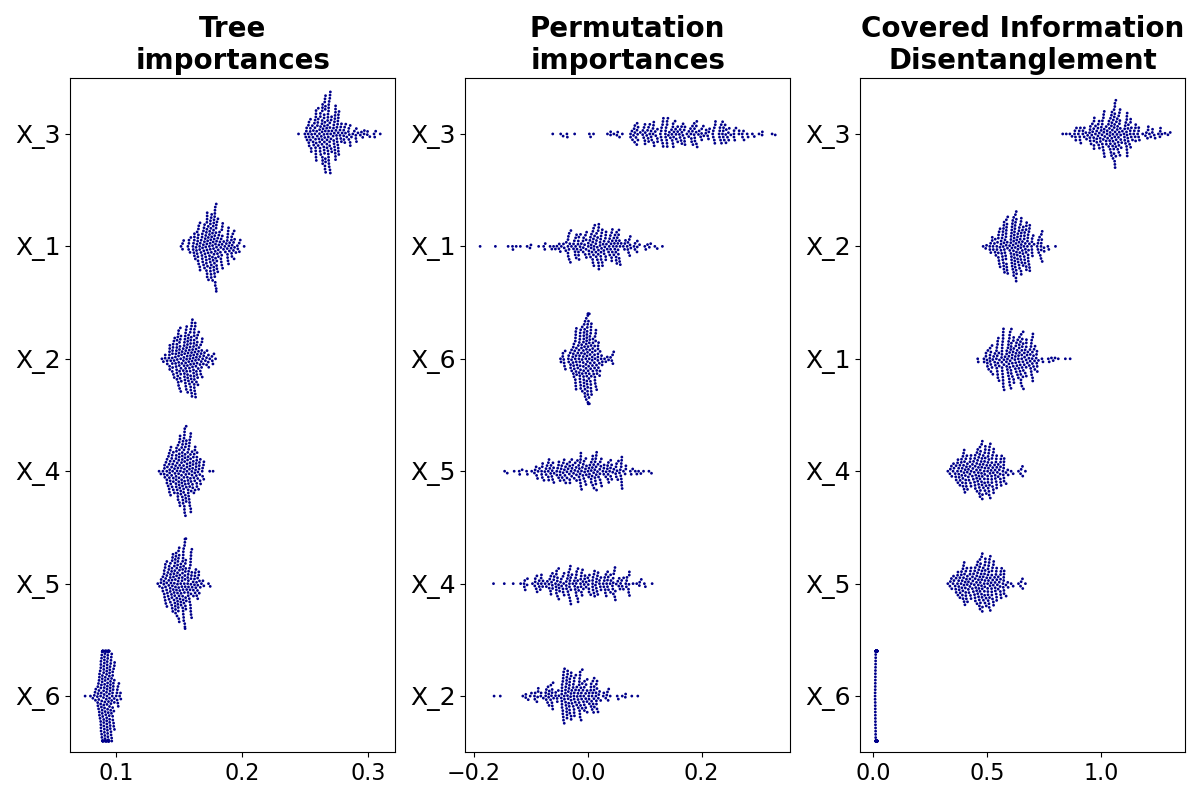}
	\caption{Comparison of the importance ranking on the multivariate gaussian dataset given by from left to right:  \textit{Tree importance}  ( \textit{Gini importance} ), \textit{\textit{permutation importance}},  \textit{CID} importance. The feature order is given by the importance median. The ground truth is $\mathcal{I}_3>\mathcal{I}_2=\mathcal{I}_1>\mathcal{I}_4=\mathcal{I}_5>\mathcal{I}_6$.}
	\label{MV_comparison_pic}
\end{figure}

\begin{algorithm}[h]
	\caption{\textit{CID} Importance}
	\textbf{Input:} $ \mathbf{X}_{M\times N}$, $\mathbf{y}$, $f$, $\Psi$, $\mathbf{G}$(optional)\\
	\textbf{Return:} $\mathcal{I}(f)$
	\begin{algorithmic}[1]
		\State $\mathbf{S}\gets$ SampleSubsets($\{1,...,M\}$)
		\State $\mathbf{e}(f)\gets$ PermutationImportance($\mathbf{X}$, $\mathbf{y}$, $f$, $\mathbf{S}$)
		
		\State $\mathbf{G}\gets$ InferGraph($[\mathbf{X},\,\mathbf{y}]$) \hfill \Comment{Infer graph if not povided}
		
		\State $\Psi_{\theta}\gets$InferMRFParams($\Psi$, $\mathbf{X}$, $\mathbf{y}$)
		\State $\mathbf{\mathbf{H}^{\wedge}}\gets$ ComputeMutualInfo($\mathbf{X}$, $\mathbf{y}$), $\mathbf{H^{\mathbf{c}}}\gets \mathbf{0}$ 
		\State $\mathcal{N}\gets$ GetNeighbors($[\mathbf{X}, y]$, $\mathbf{G}$)
		\For{$i$ \textbf{in} $[1,\,\dots,\,N]$}\Comment{can be parallelized}
		\For{$j$ \textbf{in} $[1,\,\dots,\,M]$}
		
		\State $\mathbf{d},\mathbf{e},\mathbf{F}\gets$ Potentials($\Psi_{\theta}$, $\mathbf{X},\mathbf{y}$, $i,j, \mathcal{N}_i, \mathcal{N}_y$)
		\State  $\mathbf{H}^{\mathbf{c}}_i\left[j\right] \gets$ $\mathbf{H}^{\wedge}_i\left[j\right]-log\left(f\frac{\mathbf{d}^{T}\mathbf{F}\,\mathbf{e}}{\mathbf{d}^{T}\mathbf{F}_{y}\mathbf{F}_{x_i}\mathbf{e}}\right)$
		\EndFor
		\EndFor
		\State $H^{\mathbf{c}^+}, H^{\mathbf{c}^-},H^{\wedge^+},   H^{\wedge^-}\gets$ RedundSyn($\mathbf{\mathbf{H}^{\wedge}}, \mathbf{H^{\mathbf{c}}}, \mathbf{S}$)
		\State $\phi\gets$ FitEntropyPI$\left(H^{\mathbf{c}^+}, H^{\mathbf{c}^-},H^{\wedge^+},   H^{\wedge^-}, \mathbf{e}(f)\right)$
		\State $\mathcal{I}(f)\gets \mathbf{E}_{\sim p(\mathbf{s})}\left[\phi\left(0, H^{\mathbf{c}^-},H^{\wedge^+},   H^{\wedge^-}\right)\right]$
	\end{algorithmic}
\end{algorithm}
\subsection{Cardiovascular Event Prediction with Proteomics}

\subsubsection{Problem Introduction}

Cardiovascular diseases (CVDs) are the number one cause of death globally.
Identifying asymptomatic people with the highest cardiovascular (CV) risk remains a crucial challenge in preventing their first cardiac event. Clinically used risk algorithms offer limited accuracy \cite{Piepoli}. Consequently, a substantial proportion of the general population at risk remains unidentified until their first clinical event. Hoogeveen and Belo Pereira et al. recently demonstrated increased efficacy in predicting primary events using protein-based models \cite{ehj_proteomics}. Since technical advances now allow for cheap and reproducible high-throughput proteomic analysis \cite{pea}, the field is prime for identifying new diagnostic markers or therapeutic targets, as well as developing new targeted protein panels to quickly and cheaply assess the risk of various diseases. 
The success of this endeavour is, of course, dependent on reliable feature importance identification. \\
The reason this dataset is a good candidate to test \textit{CID}, is the "biological robustness" of living systems \cite{bio_robust,cel_robust}. Biological robustness describes a property of living systems whereby specific functions of the system are maintained despite external and internal perturbations. In proteomics, robustness is achieved in two ways: since protein structure is intimately related to function \cite{protein_structure}, proteins with similar structure can exhibit similar functions, and proteins can be synthesized through different pathways in the metabolic network. This means two proteins located upstream the network relative to a third causing disease will have redundant information, and so do two proteins whose structure is similar (this is depicted in figure \ref{robustness_CID}).
\begin{figure}[tb]
	\centering
		\includegraphics[ width=1\linewidth]{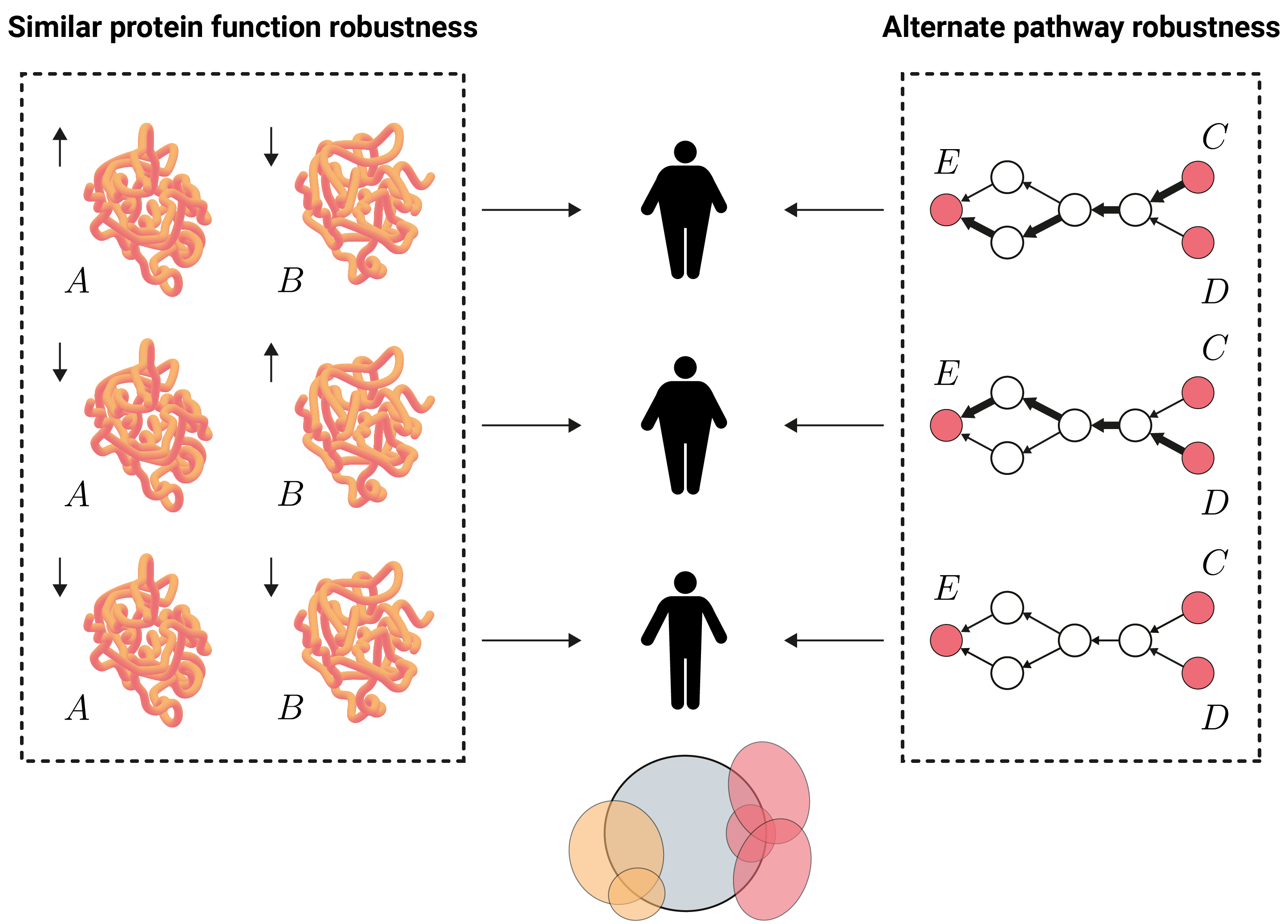}
	\caption{Illustration of biological robustness for the event prediction with proteomics problem. On the left square, it is shown how the levels of two different proteins with similar structure (and hence, similar function) impact the outcome (obesity); on the right square, it is shown how two different proteins can influence the levels of a third outcome-related one through different pathways in the metabolic network; on the bottom, there is a Venn diagram representing the information overlap of the outcome (in gray) and the other proteins considered.  }
	\label{robustness_CID}
\end{figure}

\subsection{Dataset Description}

The dataset consists of a selection of 822 seemingly healthy individuals in a nested case-control sample from the EPIC-Norfolk study \cite{epic}. Seemingly healthy individuals were defined as study participants who did not report a history of CV disease. A total of 411 individuals who developed an acute myocardial infarction (either hospitalization or death) between baseline and follow-up through 2016 were selected, together with 411 seemingly healthy individuals who remained free of any CV disease during follow-up. In the original study, the authors demonstrate how predicting short-term events leads to a significant accuracy improvement \cite{ehj_proteomics}, presumably because the proteomic profile will change over time. We used the early-event prediction dataset, where we only included patients who suffered from an event earlier than 1500 days from measurement (total of 100 patients). We do not make the code for this analysis available due to data confidentiality.

\subsection{Importance Ranking Experiment Details}

To evaluate the models' performance on days-to-event regression, we performed $100$ shuffle splits and measured the mean square error on the test set. We used 5-fold cross-validation to select the optimal hyper-parameters of a Survival Gradient Boosting regressor \cite{sksurv}. To prevent overfitting, we pre-selected 50 proteins using univariate selection. We then compared the CID with  permutation importance,  \textit{Univariate importance} ,  \textit{SAGE} \cite{SAGE}, and  \textit{Tree importance}  (\textit{Gini importance}).
We used GraphicalLasso (GL) for network inference in all our experiments and selected the $l_1$ regularization term using grid-search cross-validation. For the cardiovascular event survival analysis, we discretized the data into $10$ bins.
For this experiment we used:
\begin{alignat*}{2}
&e_i(f, \mathbf{s})&&=\phi_f\left(H_{X_i}^{\mathbf{c}^+}(\mathbf{s}), H_{X_i}^{\mathbf{c}^-}(\mathbf{s}),H_{i}^{\wedge^+}(\mathbf{s}),   H_{i}^{\wedge^-}\mathbf{s}\right) \\& &&= \mathcal{I}_i(f,(\mathbf{s}))g\left(H_{X_i}^{\mathbf{c}^+}(\mathbf{s})\right)\left(1-\frac{H_{X_i}^{\mathbf{c}^+}(\mathbf{s})}{H_{i}^{\wedge^+}(\mathbf{s})}\right),\\
&g\left(H_{X_i}^{\mathbf{c}^+}(\mathbf{s})\right)&&=\begin{cases}
c,& \text{if } H_{X_i}^{\mathbf{c}^+}(\mathbf{s})> 0,\;\;c\in\left[1, +\infty \right[\\
1,              & \text{otherwise} 
\end{cases}\;\;
\end{alignat*}
and found $c$ using grid-search on the values: $1/c=\left[ 1.2,1.4,1.6,1.8,2,2.2,2.4 \right]$. We removed data instances that contained values exceeding $4$ times the standard deviation to achieve better discretization. 
\begin{figure*}[tb]
	\centering\makebox[\textwidth][c]{
		\includegraphics[ width=.75\linewidth]{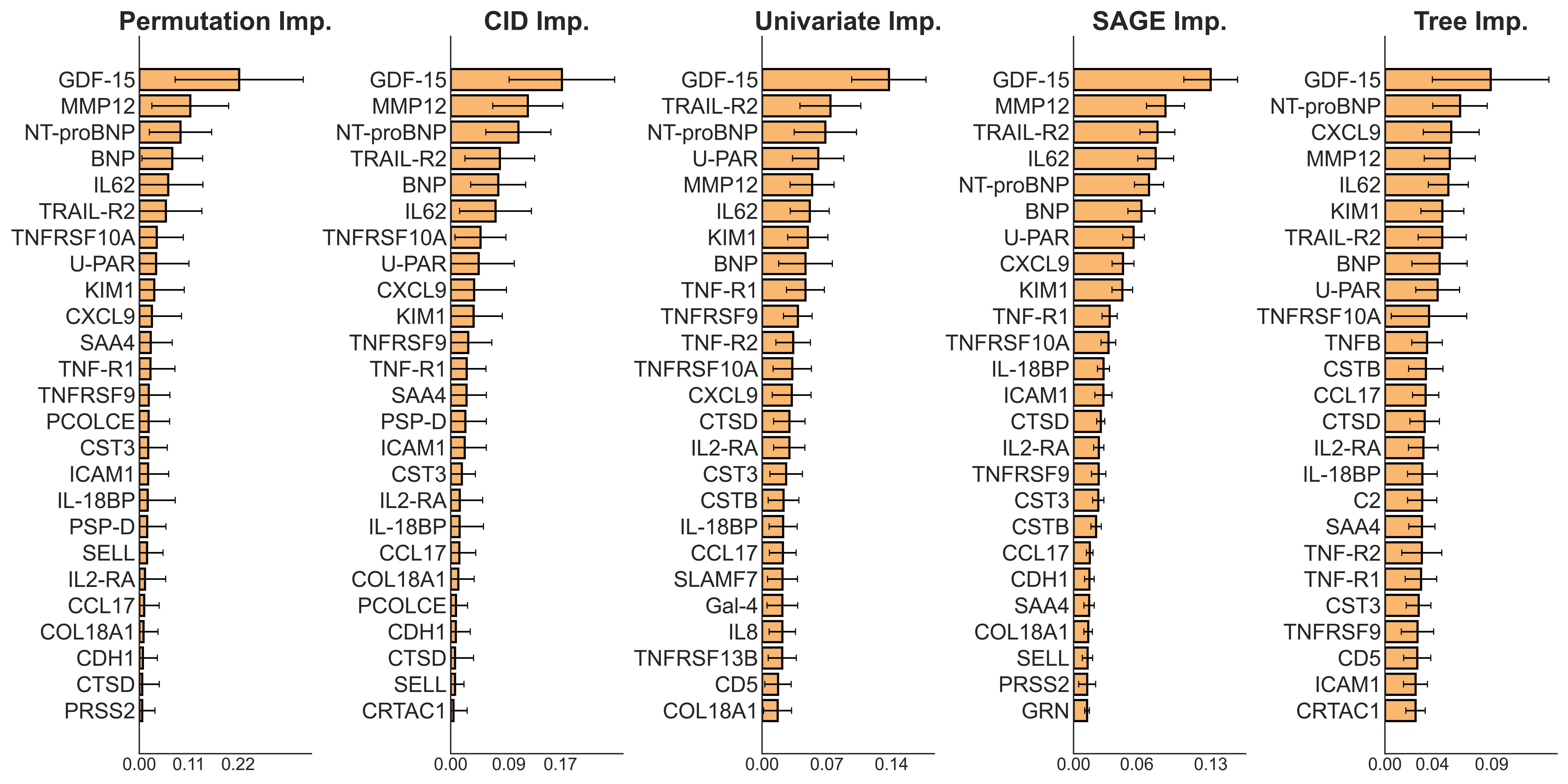}}
	\caption{Importance rankings for cardiovascular event prediction using proteomics given by \textit{permutation importance}, \textit{CID},   \textit{univariate importance}, \textit{SAGE} and  \textit{tree importance}  ( \textit{Gini importance} )}
	\label{epic_comparison}
\end{figure*}

\subsubsection{Results} 

Overall, \textit{CID} spreads the importance more evenly than Perm. imp. and aligns better with the Univariate ranking. Thus, this corroborates the hypothesis that Perm. imp. underrates correlated features. CID ranked TRAIL-R2, PSP-D, and IL2-RA significantly higher, while it ranked SELL and PCOLCE significantly lower. \\
\textbf{Gold Standard:} To establish a gold-standard analysis of the ranking, we asked world-renowned cardiovascular experts who commented on the comparison. TRAIL-R2 and GDF-15 were identified as the highest predictors of long-term mortality in patients with acute myocardial infarction in \cite{skau}. PSP-D has been identified as a strong clinical predictor of future adverse clinical outcome in stable patients with chronic heart failure in \cite{Brankovic}. Il2-RA has been positively associated with all-cause mortality, CVD mortality, incident CVD, stroke, and heart failure in \cite{durda}. To date, SELL and PCOLCE have not been associated as major players in the development of cardiovascular disease. \\
\textbf{Quantitative measure:} In order to establish a quantitative measure of the ranking quality, we followed an approach similar to what is described in \cite{SAGE}, where multiple subsets of the data were selected, the models were re-trained for each subset and then for each subset and importance method we measured the correlation between the performance and the subset's sum of importances. We also computed the model performance when trained on the top 10 to 35 proteins of each method. We also report the average running time per cycle conducted on an 8-core Intel(R) Core(TM) i7-7700HQ CPU @ 2.81Ghz. The results are displayed in table \ref{table_corr} which shows \textit{CID} outperformed the other methods on this dataset. 

\begin{table}
	\centering
	\scalebox{0.85}{
	\begin{tabular}{llll}
		\hline
		\centering Method  & Corr. &  MSE top feats &   Avg. cycle time(s)\\
		\specialrule{.1em}{.05em}{.05em} 
		Perm. Imp.       &  0.8697 &0.2824 $\pm$ 0.0107 &  0.7756  $\pm$ 0.2173 \\
		CID       & \textbf{0.8787} & \textbf{0.2801} $\pm$ 0.0098  & 18.76   $\pm$ 7.7256 \\
		Univar. Imp.   & 0.8185 & 0.2947 $\pm$ 0.0090& 0.0008  $\pm$ 0.0003    \\
		\textit{SAGE}  & 0.8499 & 0.2858$\pm$ 0.0064& 42179 $\pm$ 3835        \\
		Tree imp.   & 0.7219 & 0.2900$\pm$ 0.0087& -     \\
		\hline\\
	\end{tabular}
}
	\caption{Correlation between subset model performance and the subset's sum of importances for each method (higher is better) and the mean squared error on top 10 to 35 features for each method (lower is better), as well as the average running time per cycle in seconds.}
	\label{table_corr}
\end{table}

\section{Discussion and Conclusion}
Feature permutation is a popular algorithm used to equip black-box models with global explanations. It has the advantage of being easy to understand, but its validity suffers in the presence of covariates. We propose a novel method (CID) to disentangle the shared information between covariates and show how using Markov random fields leads to tractability, making \textit{permutation importance} competitive against methodologies where all marginal contributions of a feature are considered, such as SHAP. Due to network inference's complexity, we have only explored \textit{graphical lasso} in conjunction with Gaussian Markov random fields. Although this particular implementation is attractive for its scalability and intuitiveness, it might lack sufficient expressive power to model more complex relationships between features. \\
Recently, A. Fisher proposed \textit{model class reliance} (MCR), a method to estimate the range of variable importance for a pre-specified model class and shown how it can be computed as a series of convex optimization problems for model classes whose empirical loss is a convex optimization problem, although general computation procedures are still an open area of research \cite{fisher}. By learning a map between \textit{permutation importance} and entropy terms, the importances retrieved by \textit{CID} are less dependent on the specific fitted model than \textit{permutation importance} or SHAP, but the map quality still relies on a consistent model behavior with regards to redundant entropy, as well as a good MRF approximation to the data distribution. The former might depend on the groups of features and thus future work includes modeling this map using graph methods on the inferred network, where the node features are the entropy terms. The latter could be improved by using a class of non-parametric MRFs with higher flexibility. Should these two problems be solved, the \textit{CID} should be able to provide a truly model-agnostic feature importance method while retaining the intuitiveness of \textit{permutation importance}.
\section{Ethical Statement}

With an increasing reliance on using machine learning methods to research impactful domains such as biology and medicine, it is more important than ever to achieve model transparency and accurately determine feature relevance. In this work, we develop an efficient way to incorporate interactions when ranking variables. In the biomedical domain with thousands or millions of complex interactions among proteins, metabolites, genes, and so on, speed and correctness in determining the elements governing a given process are critical because they could significantly mitigate time, resources, and human lives lost. On the other hand, model transparency can also be exploited to develop adversarial examples or gain unwarranted access to protected systems/data.
\bibliography{CID_bib}

\end{document}